\documentclass[11pt]{article}

\newcommand{\blind}{0}

\makeatletter
\renewcommand\section{\@startsection {section}{1}{\z@}%
	{-3.5ex \@plus -1ex \@minus -.2ex}%
	{2.3ex \@plus.2ex}%
	{\normalfont\fontfamily{phv}\fontsize{16}{19}\bfseries}}
\renewcommand\subsection{\@startsection{subsection}{2}{\z@}%
	{-3.25ex\@plus -1ex \@minus -.2ex}%
	{1.5ex \@plus .2ex}%
	{\normalfont\fontfamily{phv}\fontsize{14}{17}\bfseries}}
\renewcommand\subsubsection{\@startsection{subsubsection}{3}{\z@}%
	{-3.25ex\@plus -1ex \@minus -.2ex}%
	{1.5ex \@plus .2ex}%
	{\normalfont\normalsize\fontfamily{phv}\fontsize{14}{17}\selectfont}}
\makeatother

\usepackage{amsthm,amsfonts,amssymb,amsmath,mathrsfs} 
\usepackage{pgfplotstable} 
\usepackage{graphicx,multirow} 
\usepackage{booktabs} 
\usepackage{algorithm2e}
\usepackage{tikz} 
\usepackage{url} 
\usepackage{rotating} 
\usepackage{textcomp} 
\usepackage{pgfplots} 
\usepackage{xcolor} 
\usepackage{appendix}  
\usepackage{lipsum} 
\usepackage{subcaption} 
\usepackage{diagbox}
\usepackage{longtable}
\DeclareMathSymbol{@}{\mathord}{letters}{"3B} 
\usetikzlibrary{shapes,arrows,positioning,backgrounds, calc} 
\usetikzlibrary{arrows,decorations} 
 
\usepackage{siunitx} 
 
\usepackage{tikz}  
\usepackage{pgfplots}
\usepackage{float} 
 
\newtheorem{theorem}{Theorem}
\newtheorem{lemma}{Lemma}
\newtheorem{proposition}{Proposition}
\newtheorem{corollary}{Corollary}

\theoremstyle{definition}
\newtheorem{definition}{Definition}
\newtheorem{assumption}{Assumption}
 
\theoremstyle{remark}

\definecolor{blues1}{RGB}{198, 219, 239} 
\definecolor{blues2}{RGB}{158, 202, 225} 
\definecolor{blues3}{RGB}{107, 174, 214} 
\definecolor{blues4}{RGB}{49, 130, 189} 
\definecolor{blues5}{RGB}{8, 81, 156} 
\definecolor{blues6}{RGB}{2, 34, 78} 
 
\pgfplotscreateplotcyclelist{mylist}{ 
{blues1}, 
{blues2}, 
{blues3}, 
{blues4}, 
{blues5}, 
{blues6}, 
}

\newcommand{\E}{\mathbb{E}} 
 
\newcommand{\id}[1]{\mathbf{1}_{\{{#1}\}}} 

\DeclareMathOperator*{\argmax}{arg\,max}
\usepackage{natbib}
 \bibpunct[, ]{(}{)}{,}{a}{}{,}%
\usepackage[normalem]{ulem}

\begin{document}


	\def\spacingset#1{\renewcommand{\baselinestretch}%
		{#1}\small\normalsize} \spacingset{1}
	
	\if0\blind
	{
		\title{Reoptimization Algorithms for Contextual Bandits with Knapsack Constraints}
		\author{Zhen Xu $^a$  \\
			$^a$ Management School, University of Liverpool, Liverpool, United Kingdom}
		\date{}
		\maketitle
	} \fi
	
	\if1\blind
	{
		
		\title{Reoptimization Algorithms for Contextual Bandits with Knapsack Constraints}
		\author{Author information is purposely removed for double-blind review}
		\date{}
		\maketitle
		
	} \fi
	\bigskip
	
	\begin{abstract}
	We study new algorithms for Contextual Bandits with Knapsack.  
	In these problems, there are finitely many types of customers, products, and resources.   
	Each product is made from a fixed combination of resources, and resources have finite capacity.  
	A decision maker must assign each arriving customer one out of a set of  multiple possible products.    
	Every assignment of a customer to a product will generate a random reward, which equals an unknown linear function of customer and product features, plus a noise term.  
	The objective is to jointly learn the mean reward function, and to make online assignments to minimize the expected revenue loss relative to an optimal policy that knows the reward function.  
	We propose a natural and simple extension of the Upper-Confidence-Bound (UCB) family of algorithms and apply re-optimization techniques.  
	We show that by taking advantage of re-optimization, our algorithm achieves an average regret of $O(\frac{(\ln T)^3}{T})$ where $T$ is the horizon length.  
	Our bound significantly reduces the $O(\frac{1}{\sqrt{T}})$ bound in the literature for closely related dynamic-pricing problems that are based on re-optimization.
	\end{abstract}
	
	\noindent%
	
	\spacingset{1.5} 

\section{Introduction}
\label{sec:Intro}

Digital platforms increasingly act as central matchmakers between heterogeneous demand and capacity-constrained supply. Consider a ride-hailing platform or an online advertising exchange: each day, the platform must assign millions of requests (e.g., riders, impressions) to a limited pool of resources (e.g., driver hours, advertiser budgets). The platform must make these assignment decisions in real time, yet the revenue generated by each match is rarely known in advance. Instead, it depends on complex, unobserved interactions between user characteristics and product attributes, forcing the platform to learn these values on the fly. 

Crucially, every assignment consumes scarce resources. The platform cannot blindly explore highly uncertain matches simply to learn their value, because allocating a limited resource to a suboptimal request today prevents its use for a highly profitable opportunity tomorrow. This creates an inherently coupled exploration-exploitation tradeoff tightly bound by physical or financial capacities: the platform must learn the value of different matches while meticulously pacing its limited resources over the entire time horizon.

The Contextual Bandits with Knapsacks (CBwK) framework provides a natural mathematical abstraction for this ``learning-while-allocating'' challenge \citep{Bada2013, Agrawal2016}. In a standard CBwK model, the decision maker observes a context, selects an action, and subsequently receives both a stochastic reward and a stochastic consumption of resources. While theoretically expressive, classical CBwK algorithms often assume that resource consumptions are entirely unknown and drawn from general distributions. This requires computationally heavy estimation techniques that struggle to meet the millisecond-level latency requirements of modern platforms. In reality, platform operators typically know exactly how much budget an impression will deduct, or how much capacity a specific task will require, at the exact moment of allocation. 

Motivated by this operational reality, we focus on a structured, highly practical variation of the problem: a contextual bandit setting where the resource consumption of each action is fixed and known, but the rewards remain an unknown linear function of the contexts. The objective is to jointly learn the mean reward function and make online assignment decisions to minimize expected revenue loss relative to an optimal full-information policy. Customers arriving at the platform act as the contexts, while the assignable products or services serve as the arms.

To solve this, we bridge bandit learning with classical network revenue management by proposing a UCB-guided linear programming (LP) re-solving heuristic. By periodically re-optimizing a fluid LP that is updated with the upper confidence bounds of the unknown reward parameters, our algorithm explicitly captures the opportunity cost of capacity via time-varying dual shadow prices. 

This approach contributes to the existing literature in two primary ways. First, by exploiting the specific structure of known consumptions paired with unknown linear rewards, we develop an algorithm that is computationally tractable and tailored for high-frequency resource allocation. Second, we provide rigorous theoretical guarantees for this re-solving approach. We prove that our algorithm achieves an average regret of $O(\frac{\ln^3 (T)}{T})$, where $T$ is the horizon length. This fast rate significantly improves upon the $O(\frac{1}{\sqrt{T}})$ or $O(\sqrt{\frac{\ln(T)}{T}})$ bounds prevalent in the existing CBwK and dynamic-pricing literature. To put this theoretical contribution into perspective, Table \ref{tab:literature} summarizes how our structural assumptions and algorithmic design compare to related bounds in the CBwK literature.

\begin{table}[htbp]
	\centering
	\footnotesize
	\begin{tabular}{lllc}
		\toprule
		\textbf{Reference} & \textbf{Rewards and Objective} & \textbf{Capacity Constraint} & \textbf{Regret Bound} \\
		\midrule
		\citet{Agrawal2014} & Linear reward, concave objective & Convex constraints & $O\big(\sqrt{\frac{\ln(T)}{T}}\big)$ \\
		\addlinespace
		\citet{BLS2014} & General reward and objective & Linear knapsack & $O\big(\sqrt{\frac{\ln(T)}{T}}\big)$ \\
		\addlinespace
		\citet{ADL2016} & General reward, concave objective & Linear knapsack & $O\big(\sqrt{\frac{\ln(T)}{T}}\big)$ \\
		\addlinespace
		\citet{Agrawal2016} & Linear reward objective & Linear knapsack & $O\big(\sqrt{\frac{\ln^2(T)}{T}}\big)$ \\
		\addlinespace
		\textbf{Our Work} & Linear reward objective & Linear knapsack & $\mathbf{O\big(\frac{\ln^3(T)}{T}\big)}$ \\
		\bottomrule
	\end{tabular}
	\caption{Comparison of regret bounds in the CBwK literature under proportional capacity scaling.}
	\label{tab:literature}
\end{table}

\section{Literature Review}
\label{sec:Literature}

Our problem setup and algorithmic design connect several active streams of research in operations research and computer science, specifically multi-armed bandits, contextual learning, and network revenue management.

\subsection{Bandits and Contextual Learning}
The fundamental conflict between taking actions for immediate payoff versus long-term informational gain is formalized in the multi-armed bandits (MAB) literature \citep{Thompson1933, Whittle1980, Jones1989, Mahajan2008}. \citet{Slivkins2019} provides a comprehensive overview of this vast field. 

To model the relationship between observed contexts and outcomes, the contextual bandits (CB) literature typically divides into agnostic and realizability-based approaches. Agnostic models avoid assumptions about the underlying context-outcome relationship, attempting instead to compete with the best policy in a given class \citep{Auer2003, Dudik2011, Agrawal2014}. However, these approaches often require complex cost-sensitive classification oracles to maintain computational efficiency \citep{LWJR2010}. Alternatively, realizability-based approaches assume the outcome distribution follows a known function class. When this relationship is linear, optimal regret bounds can be achieved using UCB-type algorithms \citep{YDC2011, CLLR2011}. While fundamental, these standard CB models abstract away the budget and capacity constraints that are critical to physical and financial operations.

\subsection{Contextual Bandits with Knapsacks (CBwK)}
To capture the presence of finite resources, \citet{Bada2013} introduced the bandits with knapsacks (BwK) framework, which has since been applied to dynamic pricing \citep{OmarAssaf2012}, the Adwords problem \citep{Mehta2007}, and display advertising \citep{Edwards2020}. 

Extending BwK to include contextual information naturally led to the CBwK framework. In the agnostic setting, researchers have generalized classical techniques to accommodate knapsack constraints \citep{BLS2014, ADL2016}. Yet, as \citet{Agrawal2016} demonstrate, relying on classification oracles in an agnostic CBwK setting can be NP-hard even for linear cases. Consequently, significant focus has shifted to linear realizability assumptions. By leveraging confidence ellipsoids, \citet{Agrawal2016} developed near-optimal algorithms for stochastic linear CBwK, while \citet{Slivkins2022} recently extended these ideas to smoothed contextual settings where contexts are subject to Gaussian perturbation. 

Other works provide bounds tailored to different operational assumptions. For example, \citet{AP2015} prove regret bounds that are logarithmic in the horizon but exponential in the number of resources, while \citet{NKRA2019} establish an $O(\ln T)$ competitive ratio under adversarial arrivals. For dynamic pricing with unknown demand functions, \citet{FSW2018} analyze a Thompson sampling-based algorithm yielding an $O(\sqrt{T})$ average regret. In contrast to these settings, the primary learning target in our model is the reward distribution, structured linearly over contexts, which allows us to achieve much faster convergence rates by exploiting known resource consumption matrices.

\subsection{Network Revenue Management and Re-Solving Heuristics}
From an operations perspective, our setting is a direct extension of the network revenue management (NRM) problem \citep{Gallego1994, Gallego1997}. NRM typically models dynamic resource allocation under capacity constraints, scaled asymptotically by a factor of $T$ to represent growing market sizes with constant resource scarcity \citep{Talluri2004}. 

Because exact solutions suffer from the curse of dimensionality, a widely used heuristic is the deterministic linear programming (DLP) approximation, which replaces stochastic arrivals with their expectations. Under this scaling, static DLP-based policies suffer an $O(\sqrt{T})$ revenue loss \citep{Gallego1994}. To incorporate information revealed over time and address stochastic drift, practitioners often re-optimize the DLP periodically using remaining capacities. \citet{Wang2008} show that a single endogenously timed re-solving step using probabilistic allocation drops the revenue loss to $o(\sqrt{T})$. Further, \citet{Jasin2012} prove that frequent, periodic re-solving yields an $O(1)$ revenue loss when demand distributions are fully known. When demand must be learned, \citet{Jasin2014} and \citet{ChenRoss2014} propose re-solving algorithms with $O(\ln{T})$ revenue loss. 

Our work also touches upon online knapsack and secretary problems \citep{Kley1998, Klein2005, Baba2007, Arlotto2019, Arlotto2020}. Notably, \citet{VBG2021} recently developed a re-optimization framework for single-resource stochastic allocation, achieving $O(\frac{1}{T})$ regret for online probing and $O(\frac{\ln(T)}{T})$ for distribution-agnostic knapsacks. Our approach extends this operational philosophy—frequent, LP-based re-optimization—into the multi-resource contextual bandit environment, demonstrating that maintaining tight pacing control via shadow prices is critical for overcoming the standard $\sqrt{T}$ exploration penalty.

\subsection{Paper Organization}
The paper is organized as follows. 
After the introduction and literature review in \S\ref{sec:Intro}, we describe our model and settings in \S\ref{sec:model}. 
In \S\ref{sec:Policy}, we introduce the re-optimization idea and  propose a re-solving based online learning algorithms. 
Furthermore, in \S\ref{sec:Benchmark} we also establish an LP optimization as an benchmark for any online algorithms. 
In \S\ref{sec:Dev}, we explain why our main ideas work well by introducing an expected capacity deviation analysis. 
Then in \S\ref{sec:Regret}, we analyze and bound the regret for our reoptimization algorithm.
In \S\ref{sec:Numerical}, we conduct a numerical to highlight the significance of re-solving in guaranteeing logarithmic bound performance.  
Finally, the paper concludes in \S\ref{sec:con}, and all the proofs can be found in the appendix. 

\section{Model}
\label{sec:model}
Consider a finite horizon indexed by $t$, $t=1,2,\ldots,T$. 
There is a finite set of \emph{resources} indexed by $k=1,\ldots, K$ that are not replenished over the horizon.  
Each resource $k$ has a capacity of $B_k$, $k=1,\ldots,K$.  
All the resources are known at time $1$. 

There is a finite set of \emph{products} indexed by $j=0,1,\ldots J$ that are offered.   
Each product $j$ can be described by a size-$J$ vector of features $Y_j$, $j=1,\ldots,J$.  
We assume that $||Y_j|| \le 1$, $j=1,\ldots,J$.  
We denote the set of all such feature vectors by $\mathcal{Y}$.  
Let $Y$ be a matrix  whose columns are all the $J$ feature vectors.
We assume that the products can be made instantaneously from the resources.  
Each product $j$ requires a fixed amount $a_{jk}$ of resource $k$, $j=1,\ldots,J$, $k=1,\ldots,K$.

There are $I$ customer types.  
In each period $t$, there is a known probability $\lambda_i$ that a customer of type $i$ will arrive, $i=1,\ldots,I$.  
We assume that there is at most one arrival in each period.  
Each customer $i$ is described by a size-$I$ vector of features $X_i$, $i=1,\ldots, I$.  
We denote the set of all such feature vectors by $\mathcal{X}$.   
Each customer $i$ can be satisfied with any of the products $j$, $j=1,\ldots,J$, $i=1,\ldots,I$.   
When a customer $i$ is assigned to a product $j$, a reward of $X_i^T A Y_j+\xi$ is generated, where $A$ is an unknown $I \times J$ matrix, and $\xi$ is a zero-mean random variable bounded by $1$, with variance at most $\sigma$.  
If a customer is rejected, then we can assume that the customer is assigned to a resource with infinite capacity and feature vector $0$. 
Let $X$ be a matrix  whose rows are all the $I$ feature vectors.   
We assume that the rank of $X$ is $I$ and the rank of $Y$ is $J$.   
This assumption is without loss of generality.  
Let $U=(XA)^T$.  
We normalize $U$ such that each column $||U_{i}|| \le 1$ and hence $U_{i}Y_{j} \le 1$ for all $i=1,\ldots,I$, $j=1,\ldots,J$.

\subsection{Summary of notation}
\begin{itemize}
	\item $X$ is an $I\times I$ matrix specifying the features of the $I$ customer types. 
	\item $Y$ is an $J\times J$ matrix specifying the features of the $J$ product types. 
	\item $\lambda_i$ is the probability of having customer $i$ arrive in any given period and $\bar \lambda$ is $\max(\lambda_1,\lambda_2,\cdots,\lambda_I)$. 
	\item $A$ is an unknown $M\times N$ matrix  that relates customer and product features to rewards.
	\item $a_{jk}$ is the amount of resource $k$ required by product $j$ and $C$ is $\max\limits_{j,k}{a_{jk}}$.
	\item $U$ is $(XA)^T$.  
	The $i$-th column of $U$ is $U_i$, which is uniformly bounded by $N_0$.
	\item $\hat{U}_{it}$ is our estimator for $U_i$ in period $t$. 
	\item $M_{it}$ is a $J\times J$ matrix  used to estimate $U_i$ in period $t$. 
	\item $i(t)$ the type of the customer arriving in period $t$.
	\item $j(t)$ the type of the product assigned to the customer arriving in period $t$.
	\item $r_t=\xi_t+U_{i(t)}^T Y_{j(t)}$ is the reward earned in period $t$ for assigning the customer $i(t)$ to the product $j(t)$. 
	\item ${CB}_{it}$ is a confidence ball maintained around the estimator $\hat{U}_{it}$ of $U_i$ in period $t$. 
	\item $\beta_t$ is the radius of the confidence interval ${CB}_{it}$ maintained around the estimator $\hat{U}_{it}$ of $U_i$ in period $t$. 
	\item $B_k$ is total capacity of resource $k$, and $B$ is the vector of $B=(B_1,B_2,\cdots,B_K)$.
	\item $b_k=\frac{B_k}{T}$ is average capacity of resource $k$ in each period, and $b=(b_1,b_2,\cdots,b_K)$.
	\item $b_{kt}$ is the amount of capacity of resource $k$ that algorithm $ON$ allocates to period $t$. 
	\item $x^*$  is a solution to \eqref{Key-optimization} that assigns customers to products in expectation, with perfect knowledge of the reward function. 
	\item $D$ is the feasible region of \eqref{Key-optimization}, which corresponds to vector $b$. 
	\item $\Delta$ is the minimum optimality gap for any sub-optimal extreme point for \eqref{Key-optimization}. 
	\item $x_t$ is an optimal solution to an approximation \eqref{equ:max-potential} at time $t$ of \eqref{Key-optimization} using a learned reward estimate.
	\item $x^*_t$ is an optimal solution to a variant \eqref{Key-optimization-Dt} at time $t$ of \eqref{Key-optimization} using the true reward function.
	\item $D_t$ is the feasible region of \eqref{equ:max-potential} and \eqref{Key-optimization-Dt}, which corresponds to capacity vector $b_t$.
	\item $\epsilon_{kt}$: the deviation from average of the capacity of resource $k$ that is used in period $t$.
\end{itemize}

\subsection{Stochastic formulation and optimal policy}
Let $B_{kt}$ denote the remaining capacity of resource $k$ at the beginning of period $t$ under control policy. 
Recall that the beginning of period $t$, the arriving customer, if any, reveals her type $i(t)$.  The decision maker must assign the customer to a product $j(t)$ or reject the customer (or equivalently, assign the customer to the null product $j(t)=0$).
A \emph{feasible policy} is a non-anticipatory policy that always assigns a customer to a product with sufficient resources in every period $t$.  Let $\Pi$ denote the set of feasible policies.

\begin{definition}
	\label{def:reward-path}
	For a particular realized assignment path $\{(i(s),j(s)),s=1,2,\cdots,T\}$, the reward earned under this path is defined as
	\begin{align*}
		\textrm{TR}^{path}=\sum\limits_{s=1}^{T} (U_{i(s)}^T Y_{j(s)}+\xi_s).
	\end{align*}
\end{definition}

For any policy $\pi \in \Pi$, the expected Total Revenue (TR) earned under $\pi$ is the expected sum over all possible assignment path $\{(i(s),j^\pi(s)),s=1,2,\cdots,T\}$.  
That is
\begin{align}
	\label{equ:def-reward-path}
	\textrm{TR}^{\pi}=\E \sum\limits_{s=1}^{T} U_{i(s)}^T Y_{j^\pi(s)}.
\end{align}
The stochastic formulation of the problem is given by
\begin{align}
	\label{def:opt}
	\textrm{TR}^{opt} = &\max\limits_{\pi \in \Pi} \E \sum\limits_{s=1}^{T} U_{i(s)}^T Y_{j(s)}\\ \nonumber
	\textrm{ s.t. }& \sum\limits_{s=1}^T  \sum\limits_{j=1}^{J} a_{jk} \id{j(s)=j}  \le B_k\mbox{ a.s.}, \forall k\in[1,K].
\end{align}
Note that the constraints must hold almost surely, i.e., with probability 1.

In theory, $\textrm{TR}^{opt}$ can be obtained by using dynamic programming. 
Unfortunately, the well-known curse of dimensionality makes computing the exact solution for this dynamic program intractable.

\section{Policy and Benchmark}
\label{sec:Policy}
We will describe precisely our policy and benchmark.  We start by summarizing the main ideas of the policy.
\subsection{Summary of main ideas}
Our approach is to gradually learn the columns of the matrix $U=(XA)^T$, which we denote as $U_1,U_2,\cdots,U_I$.  In the case $I=1$, a reasonable estimate of true vector $U_i$ can be constructed by minimizing the squared loss function $\sum\limits_{s=1}^t(r_s-U_1 Y_{j(s)})^2$, where $r_s$ is the reward collected at $s$.  
In this case, if we define the $J\times J$ matrix $M_t$ as $M_t=\sum\limits_{s=1}^{t} Y_{j(s)}Y_{j(s)}^T$, then the estimator is given by 
\begin{align*}
	\hat{U}=M_t^{-1} \sum\limits_{s=1}^{t} r_s Y_{j(s)}.
\end{align*}
In the case $I>1$, we similarly use $J\times J$ matrices $M_{it}=\sum\limits_{s=1}^{t} \id{i(s)=i}Y_{j(s)}Y_{j(s)}^T$ to estimate $U_i$ at time $t$, $i=1,\ldots,I$, $t=1,\ldots, T$.  

First, we define the following $2$-norm for each matrix  $M_{it}$:
\begin{align*}
	\|u\|_{2,M_{it}} =\sqrt{u^T M_{it} u}.
\end{align*}
We will maintain an estimate $\hat{U}_{it}$ of $U_i$ in period $t$, $i=1,\ldots,I$, $t=1,\ldots,T$.  
Let 
\begin{align}
	\label{equ:beta}
	\beta_t=\max\{21 J \ln(t) \ln(It^2/\delta), (1.5 \ln(It^2/\delta))^2\},
\end{align}
where $\delta$ is the parameter to control the estimate accuracy which will be determined later.  
Define a confidence ball of radius $\beta_t$ around $\hat{U}_{it}$ as follows:
\begin{eqnarray}
	\label{equ:ConfidenceBall}
	\textrm{CB}_{it}&=&\{v: \|v-\hat{U}_{it} \|_{2,M_{it}} \le \beta_t \},\quad \forall i\in[1,I],\\ \nonumber	 
	\textrm{CB}_t&=&\prod\limits_{i=1}^I \textrm{CB}_{it}.
\end{eqnarray}
In our analysis, we will show that, with our choice of $\beta_t$, $U$ always remains inside  $\textrm{CB}_t$ with high probability, for $t=1,\ldots,T$.  


Our algorithm uses the following ideas.  
We plan to use an average amount of capacity in each period.  
This average is the total capacity divided by the number of periods.  
In each period, based on our estimate of the unknown parameter, we solve an LP specific to that period, to allocate customers to resources.  
We use upper confidence bounds to deal with uncertainty in this assignment that arise from lack of knowledge about the unknown parameters.  
We use the average capacity to constrain this assignment.  
We then make the necessary assignment if any for an arriving customer, observe the realized reward, and perform a least-squares update of our parameters.  
The algorithm is a re-solving algorithm because the LPs that are solved in different periods are identical to each other, except for the amount of uncertainty surrounding the unknown parameters.

\subsection{Online Algorithm $ON$}
\begin{enumerate}
	\item Initialization: Set $b_{k1}=\frac{B_k}{T}$, $M_{i1}=I_{J \times J}$ and $\hat{U}_{i1}=0$, $i=1,2,\cdots,I$, $k=1,\ldots,K$.
	\item For $t=1$ to $T$:
	\begin{eqnarray*}
		{CB}_{it}&=&\{v: \|v-\hat{U}_{it} \|_{2,M_{it}} \le \beta_t \},\quad \forall i\in[1,I],\\
		{CB}_t&=&\prod\limits_{i=1}^I {CB}_{it}.
	\end{eqnarray*}
	
	\item In each period $t$, solve the following optimization problem to obtain probabilities $x_t$:
	\begin{align}
		\label{equ:max-potential}
		x_t =\argmax\limits_{x\in D_t} \max\limits_{V \in {CB}_t} \sum_{i,j} x_{ij} V_i^T Y_j,
	\end{align}
	where 
	\begin{align*}
		D_t=\{x\ |\ \sum_{i=1}^I \sum_{j=1}^J x_{ij} a_{jk} \le b_{kt},\ \forall k\in[1,K];\ \sum\limits_{j} x_{ij} \le \lambda_i,\ \forall i\in[1,I];\ x \ge 0\}
	\end{align*}
	is the feasible region.   
	The inner ``max'' means that we choose the best potential reward based on all the possible vectors in ${CB}_t$. Note that this program is referred to in the literature as a \emph{bilinear program} \citep{PelegMeir2008}.  Such programs can be solved by standard techniques.
	
	\item Use the solution $x_t$ to assign the customer at $t$, if any, to a resource.  
	More precisely, if the customer is of type $i$, $i=1,\ldots,I$, then assign the customer to each product $j$ with probability $\frac{x_{ijt}}{\lambda_i}$, $j=1,\ldots,J$.  
	If there are insufficient resources to make the assigned product, then reject the customer (or equivalently, assign the customer to a null resource).  The result of the assignment is a realization $(i(t),j(t))$ based on $x_t$.  
	That is, the customer at $t$ of type $i(t)$ is assigned to product $j(t)$.
	
	\item Define a random variable $d_{kt}$, whose mean equals to $b_{kt}$, to represent the realized usage of each resource $k$ in period $t$, via
	\begin{eqnarray}
		\label{equ:realization_D}
		d_{kt}&=&\left\{ \begin{array} {ll}
			a_{j(t)k}, \quad \forall k\in[1,K],\quad \mbox{if there is an arrival at $t$},\\
			0, \quad o.w.
		\end{array}\right.	
	\end{eqnarray}
	
	\item Update the capacity allocated to period $t+1$ via
	\begin{align}
		\label{equ:update_D}
		b_{k t+1}=b_{kt} - \frac{(d_{kt}-b_{kt})}{T-t}, \quad \forall k\in[1,K].
	\end{align}
	
	\item Collect the reward $r_t=\xi_t+U_{i(t)}^T Y_{j(t)}$ and update the estimator $\hat{U}_{it+1}$ and matrix $M_{it+1}$ via
	\begin{eqnarray}
		M_{i t+1}&=&\left\{ \begin{array} {ll}
			M_{i t}+ Y_{j(t)} Y_{j(t)}^T,\quad i=i(t),\\
			M_{i t},\quad o.w.
		\end{array}\right. \label{equ:update_M}\\
		\hat{U}_{it+1}&=&\left\{ \begin{array} {ll}
			M_{it+1}^{-1} \sum\limits_{s=1}^t r_s Y_{j(s)}\id{i(s)=i} ,\quad i=i(t),\\
			\hat{U}_{it},\quad o.w.
		\end{array}\right.  \label{equ:update_XA}
	\end{eqnarray}
	Under this updating rule, we have
	\begin{align*}
		\hat{U}_{it}=M_{it}^{-1} \sum\limits_{s=1}^{t-1} r_s Y_{j(s)} \id{i(s)=i}= M_{it}^{-1} (M_{it-1} \hat{U}_{it-1}+r_{t-1} Y_{j(t-1)} \id{i(t-1)=i})
	\end{align*}
	for $t\in[1,T]$.
\end{enumerate}

\subsection{Benchmark}
\label{sec:Benchmark}
We will use as benchmark an optimal solution to the following optimization problem
\begin{align}
	\label{Key-optimization}
	x^* =\argmax\limits_{x \in D} \sum\limits_{i=1}^I \sum_{j=1}^J x_{ij}  {U_i}^T Y_j,
\end{align}
where
\begin{align*}
	D=\{x\ |\ \sum\limits_{i=1}^I \sum_{j=1}^J x_{ij} a_{jk} \le \frac{B_k}{T},\ \forall k\in[1,K];\ \sum\limits_{j=1}^J x_{ij} \le \lambda_i,\ \forall i\in[1,I];\ x \ge 0\}.
\end{align*}
Note that since $U$ is unknown, the solution $x^*$ is unknown.  
The key differences between $x^*$ and $x'_t$, which is used by our online algorithm, are that 
(i) the former is constrained by the true average capacity, namely $B/T$, while the latter is constrained by the actual average remaining capacity; 
and (ii) the former is defined by the true reward parameters, namely ${U_i}^T Y_j$, while the latter is defined by an upper bound on the estimated reward parameters. 

Alternatively, the above problem can be written as
\begin{eqnarray}
	\label{eq:KKT-c}
	\max                  && \sum\limits_{i=1}^I \sum_{j=1}^J x_{ij}{U_i}^T Y_j,\\ \label{eq:tight-c}
	\textrm{ subject to } && \sum_{i=1}^I \sum_{j=1}^J x_{ij} a_{jk} \le \frac{B_k}{T}, \quad \forall k\in[1,K], \\\nonumber
	&&  \sum_{j=1}^J x_{ij} \le \lambda_i, \quad \forall i\in[1,I],\\ \nonumber             
	&&  x \ge 0.
\end{eqnarray}
We make the following key assumption about \eqref{Key-optimization}, which will enable us to quantify the regret of the online algorithm:
\begin{assumption}
	\label{ass:1}
	There is a unique dual optimal solution to \eqref{Key-optimization}.  
	Moreover, the variables $(\mu^*_1,\cdots,\mu^*_K)$ corresponding to the capacity constraints \eqref{eq:tight-c} in this solution are strictly positive.
\end{assumption}
The above assumption is similar to those made in \cite{Jasin2014} and \cite{CJD2018}.
These works study price control for network revenue management, where the objective function is the revenue rate $r(\lambda)=\lambda p(\lambda)$. 
Assuming the convexity of $r(\cdot)$ and the linearity of constraints, the resulting deterministic optimization problem becomes a convex programming problem.  
The assumption imposed by both of these papers is that the optimal vector of dual variables corresponding to the capacity constraints be strictly positive.
Remark: \cite{Jasin2014} lists two types of assumption.  
The first type is strict positivity of the dual variables, as we have assumed.
The second type is a condition A5 that requires that the static price $p_D$ be neither as low nor so high, so that $\lambda^D$ lies in a proper interior of a strict positive region.  The latter assumption is just as strict, as does not happen in general in our model, since our optimal vector $x^*_t$ might have some zero components.


The optimization problem \eqref{Key-optimization} provides us with a performance benchmark in the following sense.
\begin{proposition} 
	\label{prop:upperBound} 
	Fix an optimal solution $x^*$ of \eqref{Key-optimization}, which is an extreme point.  
	Let 
	\begin{align*}
		R^{u}= \sum\limits_{i,j} x^*_{ij} {U_i}^T Y_j.
	\end{align*}
	Then the average reward per period of any online algorithm is bounded above by $R^{u}$, and the total reward over the horizon by $R^{u}T$.
\end{proposition}

\begin{proof}[Proof of Proposition \ref{prop:upperBound}:]
	Let $OFF$ be an offline algorithm that knows a priori all demand arrivals, as well as $U$ and $Y$, and makes optimal decisions given this information.  
	On any sample path $\omega$, where the total number of arrivals over the horizon of type $i$ is $\Lambda_i(\omega)$, with $\E[\Lambda_i(\omega)]=T\lambda_i$, $OFF$ makes decisions $x_t(\omega)$ in each period $t$, to optimize
	\begin{eqnarray}
		\max                  && \sum_{i,j}\sum_{t=1}^T x_{ijt}(\omega) {U_i}^T Y_j,\\ \nonumber 
		\textrm{ subject to } && \sum_{i=1}^I \sum_{j=1}^J \sum_{t=1}^T x_{ijt}(\omega) a_{jk}\le B_k,\quad \forall k\in[1,K], \\ \nonumber
		&&  \sum_{t=1}^T\sum_{j=1}^J x_{ijt}(\omega)\le\Lambda_i(\omega),\quad \forall i\in[1,I],\\\nonumber	       
		&&  x_t(\omega) \ge 0.
	\end{eqnarray}
	Let $x(\omega)=\frac{\sum\limits_{t=1}^T x_t(\omega)}{T}$.  
	Equivalently, $OFF$ maximizes
	\begin{eqnarray}
		\max                  && \sum_{i,j} x_{ij}(\omega)  {U_i}^T Y_j,\\ \nonumber
		\textrm{ subject to } && \sum_{i=1}^I \sum_{j=1}^J x_{ij}(\omega) a_{jk}\le\frac{B_k}{T}, \quad \forall k\in[1,K], \\ \nonumber
		&&  \sum_{j=1}^J x_{ij}(\omega) \le \Lambda_i(\omega)/T, \quad \forall i\in[1,I],\\\nonumber
		&&  x(\omega) \ge 0.
	\end{eqnarray}
	Thus, $\E[x(\omega)]$ is feasible for the constraint of \eqref{Key-optimization}.  
	It follows that 
	\begin{align*}
		\E[\sum_{i,j} x_{ij}(\omega) {U_i}^T Y_j] = \sum_{i,j} \E[x_{ij}(\omega) ] {U_i}^T Y_j
	\end{align*}
	is no more than the optimal value of  \eqref{Key-optimization}, which is $R^{u}$.  
	We conclude that no online algorithm can earn more than $R^{u}$ in expectation in each period on average, or $R^{u} T$ over the entire horizon. 
\end{proof}

\section{Deviation from ideal capacity usage}
\label{sec:Dev}
\subsection{re-optimization on the perturbation of $b$}
Recall \eqref{Key-optimization}, which defines the benchmark solution $x^*$.  
In general, we cannot solve \eqref{Key-optimization} because we do not have access to $U$.  
Even if we did, the solution would not be useful because the randomness in demand will cause our capacity to deviate in a random manner.  
Nevertheless, we can adjust our use of capacity to match that of \eqref{Key-optimization}.   
More precisely, we will approximate the feasible region $D$ of \eqref{Key-optimization}, which is defined by the capacity level $B/T$, with a regions $D_t$, which is defined by the average actual capacity remaining at $t$, $t=1,\ldots,T$.

In the following variant of problem \eqref{Key-optimization}, we think of $b'=(b'_1,b'_2,\cdots,b'_K)$ as a perturbation to the  ideal average capacity $B/T$.
Consider the solution when we change $B/T$ to $b'$:
\begin{align}
	\label{Key-optimization-D}
	x'=\argmax\limits_{x\in D'} \sum\limits_{i=1}^I\sum_{j=1}^J x_{ij} {U_i}^T Y_j,
\end{align}
where, 
\begin{align*}
	D'=\{x\ |\ \sum\limits_{i=1}^I\sum_{j=1}^J x_{ij} a_{jk} \le  b'_k,\ \forall k\in[1,K];\ \sum\limits_{j=1}^J x_{ij} \le \lambda_i,\ \forall i\in[1,I];\ x^* \ge 0\}.
\end{align*}
Or alternatively,
\begin{eqnarray}
	\label{eq:KKT}
	\max                 && \sum\limits_{i=1}^I\sum_{j=1}^J x_{ij} {U_i}^T Y_j,\\  \label{eq:tight}
	\textrm{ subject to }&& \sum\limits_{i} x_{ij}  a_{jk} \le b'_k,\quad \forall k\in[1,K],\\ \nonumber
	&& \sum_{j} x_{ij} \le \lambda_i, \quad \forall i\in[1,I],\\ \nonumber
	&& x \ge 0.
\end{eqnarray}
Define the expected reward $R(b')$, which is a function of the capacity $b'$, as the optimal value of \eqref{Key-optimization-D}.  

A direct consequence of Assumption 1 is that $\frac{\partial R}{\partial B'_k}(\frac{B}{T})=\mu^*_k$ for $k=1,2,\cdots,K$.
Furthermore, we have the following proposition.
\begin{proposition}  
	\label{prop:extrmempoints}
	Let $\Delta>0$ be the minimum optimality gap for all sub-optimal extreme points of \eqref{Key-optimization}.  
	Then there exists a constant $\epsilon$ such that if $\|b'-\frac{B}{T}\|_\infty \le \epsilon$, then the followings hold:
	\begin{itemize}
		\item[i)] $\frac{\partial R}{\partial b'_k}=\mu^*_k$ for $k=1,2,\cdots,K$.
		
		\item[ii)] There is a unique optimal solution to \eqref{Key-optimization-D}.  
		
		\item[iii)] For any extreme point of $D'$ that is sub-optimal for \eqref{Key-optimization-D}, the expected reward at the point is at least $\frac{\Delta}{3}$ smaller than $R(b')$. 
		\item[iv)] For any optimal solution $x^*$ to \eqref{Key-optimization}, there exists a $IJ \times K$-matrix $Q$, such that $x'= x^* +Q(b'-\frac{B}{T})$ is an optimal extreme point of \eqref{Key-optimization-D}.  
	\end{itemize}
\end{proposition}
\begin{proof}[Proof of Proposition \ref{prop:extrmempoints}:]
	Since \eqref{Key-optimization} has a unique dual-optimal solution $\mu^*$, and $b'$ is a small perturbation of $B/T$, for sufficiently small $\epsilon$, $\mu^*$ remains dual-optimal for \eqref{Key-optimization-D}.  
	Thus, (i) and (ii) follow. 
	
	To see (iii), let $(\mu^*,\lambda^*)$ be the optimal dual solution to \eqref{Key-optimization} and \eqref{Key-optimization-D}.  
	Let $r=U^T Y$.  Choose $\epsilon$ sufficiently small such that whenever $\|b'-B/T\|_\infty <\epsilon$, 
	$$r^T W^{-1}(b'-B/T) \leq \Delta/3$$ for all primal bases $W$ of $\eqref{Key-optimization}$.  
	Note that \eqref{Key-optimization} and \eqref{Key-optimization-D} share the same set of primal bases.  
	Let $W^*$ and $W'$ be the primal bases corresponding to $x$ and $x'$, respectively.  
	Then, since $r^T W^{*-1}B/T=R(B/T)$, we conclude that $r^T W^{*-1}b'$ is within $\Delta/3$ of $R(b')$, i.e., $$|r^T W'^{-1}b' - r^T W^{*-1} b| < \Delta/3.$$  
	Let $x''$ be any suboptimal extreme point of  \eqref{Key-optimization-D} and let $W''$ be its corresponding basis.   
	Suppose that $r^T (x'-x'') <\Delta/3$.  
	Then \begin{eqnarray}
		r^T W^{*-1} b  - r^T W''^{-1} b 
		&=&(r^T W^{*-1} b  - r^T W'^{-1} b') + (r^T W'^{-1} b' - r^T W''^{-1} b')\\ \nonumber
		&&  + (r^T W''^{-1} b' - r^T W''^{-1} b) \label{eq:triangle}\\ 
		\nonumber&< & \Delta/3 + \Delta/3 + \Delta/3 \\
		\nonumber&=&  \Delta.
	\end{eqnarray}
	In other words, $W''^{-1} b$ is an extreme point with objective within $\Delta$ of optimality for \eqref{Key-optimization}.  
	This contradicts the definition of $\Delta$.
	
	To see (iv), notice that in \eqref{eq:triangle}, as $\epsilon\rightarrow 0$, the left-hand side must remain bounded below by $\Delta$, yet $r^T W^{*-1} b - r^T W'^{-1} b'\rightarrow 0$, and $r^T W''^{-1} b' - r^T W''^{-1} b\rightarrow 0$.   
	Thus, $r^T W'^{-1} b' - r^T W''^{-1} b'$ must be bounded below by a positive number $\delta$. 
	Thus, for sufficiently small $\epsilon$, $|r^T W^{*-1} b'  - r^T W'^{-1} b'| \leq  |r^T W^{*-1} b' - r^T W^{*-1} b| + |r^T W^{*-1} b - r^T W'^{-1} b'| \le r^T W'^{-1} b' - r^T W''^{-1} b'$, for any suboptimal $x''=W''^{-1} b'$.  
	It follows that $W^{*-1} b'$ must be optimal for  \eqref{Key-optimization-D}, or $x'=W^{*-1} b'$.  
	In other words, 
	\begin{align*}
		x-x'=W^{*-1} b-W^{*-1} b'=W^{*-1} (b-b').
	\end{align*}
	Thus, (iv) follows by taking $Q=W^{*-1}$. 
\end{proof}

The following lemma shows how large we can choose $\epsilon$ in Proposition~\ref{prop:extrmempoints}.

\begin{lemma}
	\label{lemma:mu}
	The $\epsilon$ in Proposition~\ref{prop:extrmempoints} is at least in the order of $\Theta(\Delta)$.
	To be precise, $\epsilon$ can be chosen at least at $\frac{\Delta}{K C_{\mu}}$, where $C_{\mu}=\max\limits_{k} \{\min\limits_{i,j} \{\frac{U_i^T Y_j}{a_{jk}}\}\}$.
\end{lemma}

In our online algorithm $ON$, the updated capacity vector $(b_{1t},b_{2t},\ldots,b_{kt})$ from the last period $t-1$, will define the feasible region $D_t$ in the re-optimization \eqref{equ:max-potential}. In order to apply Proposition~\ref{prop:extrmempoints} in $D_t$, we have the following definition.

\begin{definition}
	Fixing the $\epsilon=\frac{\Delta}{K C_{\mu}}$ in Lemma~\ref{lemma:mu} and Proposition~\ref{prop:extrmempoints}, and we call the region $D_t$ in  \eqref{equ:max-potential} is \emph{well-approximated} if the actual average remaining capacity is sufficiently close to the ideal average, i.e., $|b_{kt}-\frac{\bar{B_k}}{T}| \le \epsilon, \forall k\in[1,K].$  
\end{definition}

We denote $x^*_t$ as the optimal solution in the feasible region $D_t$ at time $t$ in the following way.
\begin{align}
	\label{Key-optimization-Dt}
	x^*_t =\argmax\limits_{x\in D_t} \sum_{i,j} x_{ij}  {U_i}^T Y_j,
\end{align}
where $D_t=\{x |\sum\limits_{i} \sum\limits_{j} x_{ij} a_{ik} \le  b_{kt}, \forall k\in[1,K]; \sum\limits_{j} x_{ij} \le \lambda_i, \forall i\in[1,I]; x \ge 0\}$. 
Then when $D_t$ is well-approximated, the assumptions of Proposition~\ref{prop:extrmempoints} apply, and we can claim that each optimal solution to \eqref{Key-optimization-Dt} has a unique corresponding solution to \eqref{Key-optimization}.
Without loss of generality, we fix an optimal solution $x^*$ to \eqref{Key-optimization}.
When $D_t$ is well-approximated, each optimal solution to \eqref{Key-optimization-Dt} has the closed-form $x^*_{t}= x^* + Q(b'-\frac{B}{T})$.  
We also know that for other extreme points to \eqref{Key-optimization-Dt}, the optimality gap is either zero or at least $\frac{\Delta}{3}$.

\subsection{Expected reward in re-optimization}
Note that the updating of $b$ in Online Algorithm $ON$ is 
\begin{align*}
	b_{kt+1}=b_{kt} -\frac{d_{kt}-b_{kt}}{T-t}.
\end{align*}
By induction, we can conclude that $$b_{kt+1}=\frac{B_k}{T}-\sum\limits_{s=1}^{t} \frac{d_{ks}-b_{ks}}{T-s}.$$

Let $$\eta_{kt}=b_{kt+1}-\frac{B_k}{T}=\sum\limits_{s=1}^{t} \frac{d_{ks}-b_{ks}}{T-s}$$
be the deviation in $b$ in each period and $\eta_{kt}$ are martingales.
Recall the $\epsilon$ in Lemma~\ref{lemma:mu} and define 
\begin{align}
	\label{equ:tau}
	\tau:= \inf\{t\ |\ \exists k\in [1,K]\ \textrm{ s.t. } \eta_{kt} > \epsilon\}.
\end{align}
Then $\tau$ is a stopping time.
Let $\rho_0=\sum\limits_{i,j} x^*_{ij} {U_i}^T Y_j=R^{u}$ be the optimal average reward and 
\begin{equation}
	\label{equ:R}
	\rho_{t+1}=\left\{ \begin{array} {ll}
		\sum\limits_{i,j} x^*_{ijt} {U_i}^T Y_j, & \forall t\in[1,\tau-1],\\
		\rho_{t}, &\quad \forall t\in[\tau,T].
	\end{array}\right.
\end{equation} 
It's easy to see that $\rho_{t}$ equals the reward under applying policy $x^*_t$ with constraint $D_t$ before the stopping time $\tau$ and remains same after $\tau$.
That means $\rho_t$ can be written as the form of $\sum\limits_{i,j} x_{ij} {U_i}^TY_j$ for all $t$.
Note that ${U_i}^TY_j \le 1$ hence we have a uniform bound $1$ on $\rho_t$ for all $t$.

By Proposition~\ref{prop:extrmempoints}, we can fix an optimal solution $x^*$ and a corresponding matrix $Q$ such that if $D_t$ is well-approximated, then the optimal solution to \eqref{Key-optimization-Dt} has the form $x^*_t=x^*+Q \eta_{t}$, where $\eta_{t}=(\eta_{1t},\eta_{2t},\cdots,\eta_{Kt})$.
Building on this fact, $\rho_{t+1}=\sum\limits_{i,j} x^*_{ijt} {U_i}^T Y_j=R^{u}+ \sum\limits_{i,j} (Q \eta_{t})_{ij} {U_i}^TY_j$ .
Since $\{\eta_{kt},t=1,2,\cdots\}$ are martingales for all $k=1,\ldots,K$, $\{\rho_{t},t=1,2,\cdots\}$ are also martingales, which leads to the following proposition.
\begin{proposition}
	\label{prop:extrmempoints-average}
	For any stopping time $\tau'$, $\E \rho_{\tau'}=R^{u}$.  
	It follows that $\sum\limits_{s=1}^T \rho_{s} = R^{u}T$.
\end{proposition}

\section{Regret of the Online Algorithm}
\label{sec:Regret}
In this section, we will analyze the regret of the online algorithm.
The main theorem of this paper is stated below.
\begin{theorem}
	\label{thm:main}
	The regret of the online algorithm $ON$ is at most $1+I+\frac{K^3 C^2 C_{\mu}^2}{\Delta^2}+\frac{K^3 C^2 C_{\mu}^2}{\Delta^2} \ln(T)+\frac{252IJ^2}{\Delta} \ln^3(T)$. 	
\end{theorem}

To analyze the regret, We first define a pseudo regret as the difference in the expected reward between the ideal solution $x_t^*$ and the approximate solution $x_t$ in each period $t$. 
Actually, the pseudo regret is only coming from the deviation in $b$.

\subsection{Bounding the total regret by deviation}
\begin{definition}
	Define the pseudo regret $g_t$ in each period $t$ as
	\begin{align*}
		g_t&=\sum\limits_{i,j} (x_t^*)_{ij}  {U_i}^T Y_j- \sum\limits_{i,j} (x_t)_{ij} {U_i}^T Y_j \\
		&=\sum\limits_{i,j} x^*_{ijt}  {U_i}^T Y_j- \sum\limits_{i,j} x_{ijt} {U_i}^T Y_j.
	\end{align*}
\end{definition}

Note that the pseudo regret $g_t$ is the difference in expected reward in period $t$ between the ideal solution $x_t^*$ and the approximate solution $x_t$, both coming from the feasible region $D_t$. 
By the definition of $\tau$ in \eqref{equ:R}, after the stopping time $\tau$, the optimal solution to \eqref{Key-optimization-Dt} at $\tau$ is $x^*_{ij\tau}$. 
So $\rho_{\tau} = \sum\limits_{i,j} x^*_{ij\tau} {U_i}^TY_j \le \sum\limits_{i,j} x^*_{ij\tau}=1$, and $\rho_t=\rho_{\tau} \le 1$, for all $t >\tau$.
Before the stopping time $t \le \tau$, $D_t$ is well-approximated.  Now depending on $t \le \tau$ or $t > \tau$, we can bound the regret in the following proposition.
\begin{proposition}
	\label{prop:total-loss}
	The regret of the online algorithm $ON$ is bounded by $\E \sum\limits_{s=1}^{\tau} g_s + \E [T-\tau] $.
\end{proposition}
\begin{proof}[Proof of Proposition \ref{prop:total-loss}:]
	As noted in Proposition \ref{prop:upperBound}, the optimal reward of any algorithm is bounded by $R^{u}T$.   
	Then the difference in rewards between the online policy $ON$ and an optimal policy is bounded above by 
	\begin{eqnarray}
		R^{u}T - \E \sum\limits_{s=1}^{T} \sum\limits_{i,j} x_{ijt} {U_i}^TY_j  \nonumber
		&\le& \E \sum\limits_{s=1}^T \rho_{s} - \E \sum\limits_{s=1}^{\tau} \sum\limits_{i,j} x_{ijt} {U_i}^TY_j  \\\nonumber
		&=& \E \sum\limits_{s=1}^{\tau} \rho_{s} - \E \sum\limits_{s=1}^{\tau} \sum\limits_{i,j} x_{ijt} {U_i}^TY_j   +\E \sum\limits_{\tau+1}^T \rho_{s} \\\nonumber
		&\le& \E \sum\limits_{s=1}^{\tau} (\sum\limits_{i,j} x^*_{ij} {U_i}^TY_j - \sum\limits_{i,j} x_{ijt}  {U_i}^T Y_j) +\E \sum\limits_{\tau+1}^T \rho_{s} \\\nonumber
		&\le &\E \sum\limits_{s=1}^{\tau} g_s + \E [T-\tau]\nonumber.
	\end{eqnarray} 
\end{proof}
Then in the following two subsections~\ref{subsec:learn} and \ref{subsec:resolve}, we will bound each term in Proposition~\ref{prop:total-loss}, respectively. 

\subsection{Regret on the estimator of $U$}
\label{subsec:learn}
The following proposition states that if $U$ is in $\textrm{CB}_{t}$, then the difference in reward for each product is also small at time $t$.
\begin{proposition}
	\label{prop:error_u}
	If $U \in \textrm{CB}_{t}$, then $|(\hat{U}_{it}-U_i)^T Y| \le \sqrt{\beta_t Y^T M_{it}^{-1} Y}$ for any vector $Y$, $i=1,\ldots,I$, $t=1,\ldots,T$.  
\end{proposition}
\begin{proof}[Proof of Proposition \ref{prop:error_u}:]
	In this proof, we omit the superscript $i$ for simplicity. 
	\begin{eqnarray}
		|(\hat{U}_t-U)^T Y|\nonumber &=&|(\hat{U}_t-U)^T M_t^{1/2} M_t^{-1/2} Y|\\\nonumber
		&=&|(M_t^{1/2}(\hat{U}_t-U))^T  M_t^{-1/2} Y|\\\nonumber
		&\le&||(M_t^{1/2}(\hat{U}_t-U))^T|| \cdot || M_t^{-1/2} Y||\\\nonumber
		&\le&\sqrt{\beta_t Y^T M_t^{-1} Y}. \nonumber
	\end{eqnarray}
	Since $U$ is in $\textrm{CB}_{t}$, the last line holds. 
\end{proof}

Then, the following proposition shows that the pseudo  regret $g_t$ can be bounded by a function of the decision $x_t$ alone.
\begin{proposition}
	\label{prop:regret-g}  
	If $U \in \textrm{CB}_{t}$, the incremental regret $g_t$ satisfies
	\begin{align*}
		g_t\le 2 \sum\limits_{i,j} x_{ijt} \sqrt{\beta_t} \min(\sqrt{Y_j^T M_{it}^{-1} Y_j},1).
	\end{align*}  
\end{proposition}
\begin{proof}[Proof of Proposition \ref{prop:regret-g}:]
	Let $\hat{V}$ be the inner maximizer of \eqref{equ:max-potential} in period $t$.  Then $\hat{V}_i \in \textrm{CB}_{it}, i=1,2,\cdots,I$. 
	We have
	\begin{align*}
		\sum\limits_{i,j} x^*_{ijt} {U_i}^T Y_j \le \sum\limits_{i,j} x_{ijt} \hat{V}_i^T Y_j.
	\end{align*}
	Moreover,
	\begin{eqnarray}
		\sum\limits_{i,j} x^*_{ijt} {U_i}^T Y_j - \sum\limits_{i,j} x_{ijt} {U_i}^T Y_j\nonumber 
		&\le& \sum\limits_{i,j}  x_{ijt} {\hat{V}_i}^T Y_j - \sum\limits_{i,j}  x_{ijt} {U_i}^T Y_j\\\nonumber
		&=&\sum\limits_{i,j}  x_{ijt} {\hat{V}_i}^T Y_j - \sum\limits_{i,j}  x_{ijt} \hat{U}_{it}^T Y_j \\\nonumber
		& & +\sum\limits_{i,j}  x_{ijt} \hat{U}_{it}^T Y_j-\sum\limits_{i,j}  x_{ijt} {U_i}^T Y_j \\\nonumber
		&\le&\sum\limits_{i,j}  x_{ijt} |\hat{V}_i-\hat{U}_{it}|^T Y_j +\sum\limits_{i,j}  x_{ijt} |\hat{U}_{it}-{U_i}|^T Y_j \\\nonumber
		&\le& 2 \sum\limits_{i,j}  x_{ijt} \sqrt{\beta_t} \sqrt{Y_j^T M_{it}^{-1} Y_j}. \nonumber
	\end{eqnarray}
	Since $U_i$ and $\hat{V}_i$ are in ${CB}_{t}$, the last inequality holds.
	
	Note that $g_t=\sum\limits_{i,j} x^*_{ijt}  {U_i}^T Y_j- \sum\limits_{i,j} x_{ijt} {U_i}^T Y_j \le \sum\limits_{i,j} (|x^*_{ijt}|+|x_{ijt}|) |{U_i}^T Y_j| \le \sum\limits_{i,j} (x^*_{ijt}+x_{ijt}) \le 2\sqrt{\beta_t}$.
	By the inequality proved in the previous paragraph, the incremental regret $g_t$ in each period $t$ is bounded by
	\begin{align*}
		g_t \le  2 \sqrt{\beta_t} \sum\limits_{i,j} x_{ijt} \min(\sqrt{Y_{j}^T M_{it}^{-1} Y_{j}},1).
	\end{align*} 
\end{proof}
The updating rules \eqref{equ:update_M} and \eqref{equ:update_XA} reveal how $M_{it}$ evolves over time, and we have the following two lemmas which give bounds on $\det(M_{it})$ and $\sum g_t^2$.

\begin{lemma}
	\label{lemma:detwhole}
	For all $i=1,\ldots,I$ and $t=1,\ldots,T$, $\det(M_{it}) \le t^J$.
\end{lemma}
	\begin{lemma}
		\label{lemma:g-to-M} 
		If $\forall s\in [1,t], U \in \textrm{CB}_{s}$, then $\sum\limits_{s=1}^{t} g_s^2 \le \frac{8}{3}\beta_t IJ\ln(t)$.
	\end{lemma}
	
	By Proposition~\ref{prop:extrmempoints}, at each period $s$, either $g_s=0$ or $g_s \ge \frac{\Delta}{3}$, assuming that all regions $D_s$ are well-approximated.
	Thus, $\sum\limits_{s=1}^t \Delta g_s \le 3\sum\limits_{s=1}^t g_s^2 \le 12\beta_t IJ \ln(t)$. 
	Thus, we have the following proposition which bounds the first term of the regret in Proposition~\ref{prop:total-loss}.
	\begin{proposition}
		\label{prop:regret-u}
		If each $D_s$ is well-approximated and $U \in \textrm{CB}_{s}, \forall 1 \le s \le t$, then $\sum\limits_{s=1}^t g_s \le \frac{12 IJ \beta_t \ln(t)}{\Delta}$. 
	\end{proposition}
	
	\subsection{Regret in resolving}
	\label{subsec:resolve}
	Recall $\{\eta_{kt}\}$ is a martingale for each $k=1,\ldots,K$. 
	Therefore, by Kolmogorov's inequality, we have 
	\begin{align}
		\label{equ:Kolmogorov}
		\Pr(\max\limits_{1 \le s \le t} |\eta_{ks}| \ge \epsilon) \le \frac{1}{\epsilon^2} \mathrm{Var}(\eta_{kt}) \le \frac{C^2}{\epsilon^2} \sum\limits_{s=1}^{t} \frac{1}{(T-s)^2},
	\end{align}
	where $C:=\max\limits_{1\le j \le J,\ k=1,\ldots,K}{a_{jk}}$ is a upper bound of the variance of $d_{kt}-b_{kt}$ for any $t$ and $k$.  
	Then 
	\begin{align}
		\label{equ:Kolmogorov-n}
		\Pr(\tau \le t) \le \sum\limits_{k=1}^{K} \Pr(\max\limits_{1 \le s \le t} |\eta_{ks}| \ge \epsilon) \le \frac{K C^2}{\epsilon^2} \sum\limits_{s=1}^{t} \frac{1}{(T-s)^2}.
	\end{align}
	Hence, the expected number of periods remaining after the stopping period $\tau$ can be bounded in the following proposition.
	\begin{proposition}
		\label{prop:resolve-tau}
		$\E[T-\tau] \le 1+\frac{K C^2}{\epsilon^2}+\frac{KC^2}{\epsilon^2} \ln(T)$.
	\end{proposition}
	\begin{proof}[Proof of Proposition \ref{prop:resolve-tau}:]
		We can rewrite the expected $\tau$ as
		\begin{eqnarray}
			\E[\tau] \nonumber
			&=&\sum\limits_{t=1}^{T} \Pr(\tau\ge t)\\\nonumber
			&=&T-\sum\limits_{t=1}^{T} \Pr(\tau<t) \\\nonumber
			&=&T-1-\sum\limits_{t=1}^{T-1} \Pr(\tau<t).\nonumber
		\end{eqnarray}
		In other words, 
		\begin{eqnarray}
			\E[T-\tau] \nonumber
			&\le&1+ \sum\limits_{t=1}^{T-1} \Pr(\tau<t)\\\nonumber
			&\le&1+ \sum\limits_{t=1}^{T-1} \frac{K C^2}{\epsilon^2} \sum\limits_{s=1}^{t} \frac{1}{(T-s)^2} \\\nonumber
			&=&1 + \frac{KC^2}{\epsilon^2} \sum\limits_{t=1}^{T-1} \frac{T-t}{(T-t)^2} \\\nonumber
			&=&1 + \frac{KC^2}{\epsilon^2} \sum\limits_{t=1}^{T-1} \frac{1}{t} \\\nonumber
			&\le&1+\frac{KC^2}{\epsilon^2}+\frac{KC^2}{\epsilon^2} \ln(T), \nonumber
		\end{eqnarray}
		where the first inequality comes from \eqref{equ:Kolmogorov-n}. 
	\end{proof}
	
	\subsection{Concentration of $U$}
	\label{sec:U}
	In this subsection, we show that with high probability, the true parameter $U$ is in $\textrm{CB}_{t}$ for all $t=1,\ldots,T$, which is the following proposition.
	\begin{proposition}	
		\label{prop:true-parameter}
		With the choice of $\beta_t$, we have that $\Pr(U \in \textrm{CB}_{t},\ \forall t=1,\ldots,T) \ge 1-\delta$. 
	\end{proposition}
	In order to prove Proposition~\ref{prop:true-parameter}, we define $z_{it}=M_{it}(\hat{U}_{it}-U_i)$.
	Note that $\|\hat{U}_{it}-U_i\|_{2,M_{it}}=\|z_{it}\|_{2,M_{it}^{-1}}$, hence $U \in \textrm{CB}_{t}$ is equivalent to $\|z_{it}\|_{2,M_{it}^{-1}} \le \beta_t$.  
	
	\begin{lemma}
		\label{lemma:z}
		\begin{align*}
			z_{it}= \sum\limits_{s=1}^{t} \xi_s Y_{j(t)} \id{i(s)=i} - U_i.
		\end{align*}
	\end{lemma}
	Let $S_{it}=\sum\limits_{s=1}^{t} \xi_s Y_{j(s)} \id{i(s)=i}$, then $z_{it}=S_{it}-U_i$.
	Recall that $\xi_t$ is a zero-mean random variable bounded in $[-1,1]$, hence $S_{it}$ is a martingale.
	In our analysis, we use the following lemma to show that with high probability the martingale $S_{it}$ stays close to zero, due to \cite {YDC2011}.
	
	\begin{lemma}
		\label{lemma:U-con}
		For any $\delta >0$, with probability at least $1-\delta$, for all $t \ge 1$.
		\begin{align*}
			\|S_{it}\|_{2,M_{it}^{-1}} \le \sqrt{2 \log(\frac{\det(M_{it})^{1/2}}{\delta})}
		\end{align*}
	\end{lemma}
	\begin{proof}[Proof of Lemma \ref{lemma:U-con}:]
		Note that $M_{i1}=I_{J \times J}$ and $\xi_t$ is bounded in $[-1,1]$, by applying and setting $R=1$ in Theorem~1 of \cite{YDC2011}, we get the desired result. 
	\end{proof}
	By Lemma~\ref{lemma:detwhole}, we have that $\det(M_{it}) \le t^J$, and hence $\|S_{it}\|_{2,M_{it}^{-1}} \le \sqrt{\log(\frac{t^J}{\delta^2})}$.
	Noticing that $\|U_{i}\|_{2,M_{it}^{-1}} \le \|U_{i}\|_{2,M_{0}^{-1}}= \|U_{i}\| \le 1$ and using the triangle inequality, we can bound 
	$\|z_{it}\|_{2,M_{it}^{-1}}$ by $\|S_{it}\|_{2,M_{it}^{-1}}+\|U_{i}\|_{2,M_{it}^{-1}} \le \sqrt{\log(\frac{t^J}{\delta^2})}+1$.
	This leads to $\|z_{it}\|_{2,M_{it}^{-1}} \le \beta_t$ which completes the proof of Proposition~\ref{prop:true-parameter}.
	Now we are ready to prove Theorem~\ref{thm:main}.
	\begin{proof}[Proof of Theorem \ref{thm:main}:]
		Define event $S^*=\id{U \in \textrm{CB}_t,\ \forall t=1,\ldots,T}$ of having the true $U$ in each of our Confidence Balls $\textrm{CB}_t$, $t=1,\ldots,T$.  
		Conditioning on $S^*$ and applying Proposition~\ref{prop:regret-u}, we have
		\begin{eqnarray}
			\E \sum\limits_{s=1}^{\tau} g_s \nonumber
			&=& \Pr(S^*=1) \E\left[\sum\limits_{s=1}^{\tau} g_s | S^*=1 \right] 
			+ \Pr(S^*=0) \E\left[\sum\limits_{s=1}^{\tau} g_s | S^*=0 \right]\\\nonumber
			&\le& (1-\delta) \frac{12 IJ \beta_T \ln(T)}{\Delta} + \delta T \le \frac{252IJ^2}{\Delta} \ln^3(T)+I,\nonumber
		\end{eqnarray}
		where the last two inequalities are from taking $\delta=\frac{I}{T}$ in Proposition~\ref{prop:true-parameter} and \eqref{equ:beta}.
		Plugging the $\epsilon=\frac{\Delta}{K C_{\mu}}$ in Lemma~\ref{lemma:mu} into Proposition \ref{prop:resolve-tau}, then $\E[T-\tau]$ can be bounded by $1+\frac{K^3 C^2 C_{\mu}^2}{\Delta^2}+\frac{K^3 C^2 C_{\mu}^2}{\Delta^2} \ln(T)$.
		And by Proposition~\ref{prop:total-loss}, we can show that regret of the online algorithm $ON$ is at most $1+I+\frac{K^3 C^2 C_{\mu}^2}{\Delta^2}+\frac{K^3 C^2 C_{\mu}^2}{\Delta^2}  \ln(T)+\frac{252IJ^2}{\Delta} \ln^3(T)$. 
	\end{proof}

\section{Numerical Experiments}
\label{sec:Numerical}
we consider a network revenue management problem with multiple resources. 
There are $3$ types of customers and $5$ types of products, and $4$ types of resources ($I=3,\ J=5,\ K=4$). 
The customer features are of dimension $3$ and are independently drawn from the uniform distribution on interval $[0,1]$.
The product features are of dimension $5$ and are independently drawn from the uniform distribution on interval $[0,1]$.
Set the arrival rates $\mu_i=1$, $i=1,\ldots,I$, the arrivals of
different classes are independent and the capacity of the resources are $T*b$, where the vector of the average capacities per unit time is given by $b=(1,1,1,1)$.
The coefficients $a_{jk}$, $j=1,\ldots,5; k=1,\ldots,4$, form the bill-of-materials matrix as follows:
\[
\begin{bmatrix}
	1 & 0 & 1 & 0 & 0 \\
	0 & 1 & 0 & 1 & 1 \\
	1 & 1 & 0 & 0 & 0 \\
	0 & 0 & 0 & 0 & 1
\end{bmatrix}
\]
And for the payoff matrix $A$, we consider two cases: 
\[
A=\begin{bmatrix}
	10 & 3 & 6 & 1 & 2 \\
	4 & 5 & 7 & 1 & 8 \\
	3 & 1 & 0 & 1 & 5 
\end{bmatrix} 
\textbf{ and }
A=\begin{bmatrix}
	15 & 4 & 9 & 2 & 3 \\
	4 & 5 & 7 & 1 & 8 \\
	3 & 1 & 0 & 1 & 5 
\end{bmatrix}
\]

We compare three following algorithms based on different learning and optimization schemes. The horizon length is $T = 10000$ periods. 
Each problem instance is simulated $1000$ times.
\begin{enumerate}
	\item Algorithm Re-UCB: The proposed algorithm $ON$.
	\item Algorithm UCB: Online updating without re-solving (see Algorithm in Appendix \ref{sc:non-UCB}).
	\item Algorithm Re-SEP: Separated learning and exploitation with re-solving (see Algorithm in Appendix \ref{sc:Re-SEP}).
	\item Algorithm SEP: Separated learning and exploitation without re-solving (see Algorithm in Appendix \ref{sc:SEP}).
\end{enumerate}

The performance metric is expected regret compared to the optimal reward, and the results are showed in Figure \ref{fig:loss}.

\pgfplotstableread{sample3.txt}{\mydata} 
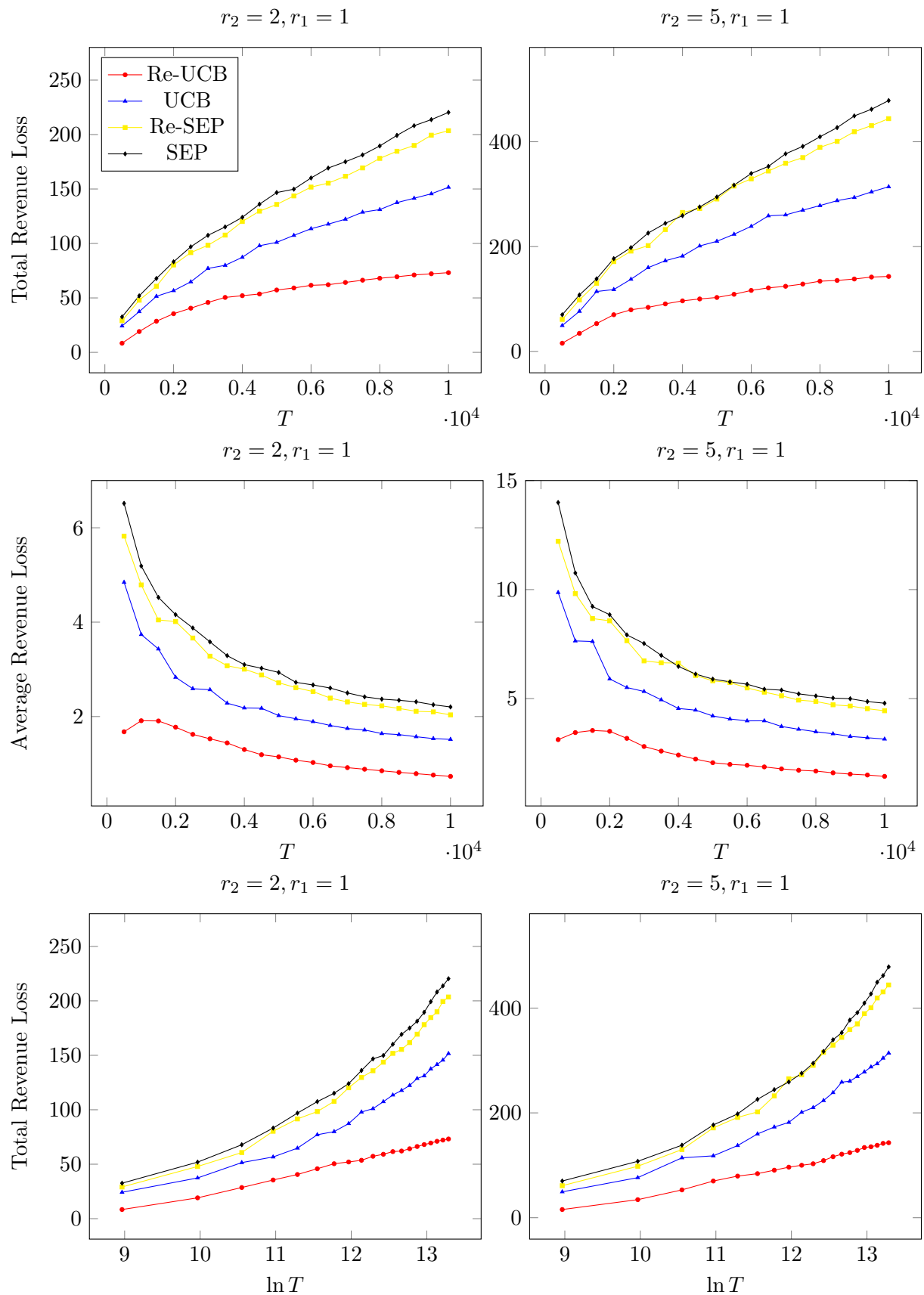
\begin{figure}[H] 
\centering 
\begin{subfigure}[b]{1.0\textwidth} 
		\begin{tikzpicture} 
			\begin{axis}[ 
				xlabel=$T$, ylabel=Total Revenue Loss,
				ymax= 280, legend pos=north west, title={$r_2=2, r_1=1$} 
				]    
				\addplot [  
				mark=*, mark size=1pt,  color= red
				] 
				table 
				[ 
				x expr=\thisrowno{0}, y expr=\thisrowno{1}/100 
				] {sample3.txt}; 
				\addplot [  
				mark=triangle*, mark size=1pt,  color= blue
				] 
				table 
				[ 
				x expr=\thisrowno{0}, y expr=\thisrowno{2}/100 
				] {sample3.txt};
				\addplot [  
				mark=square*, mark size=1pt,  color= yellow
				] 
				table 
				[ 
				x expr=\thisrowno{0}, y expr=\thisrowno{3}/100  
				] {sample3.txt}; 
				\addplot [  
				mark=diamond*, mark size=1pt,  color= black
				] 
				table 
				[ 
				x expr=\thisrowno{0}, y expr=\thisrowno{4}/100 
				] {sample3.txt};        
				\legend{Re-UCB, UCB, Re-SEP, SEP}
			\end{axis} 
		\end{tikzpicture}\hfill%
		\begin{tikzpicture} 
			\begin{axis}[ 
				xlabel=$T$, 
				ymax= 580, legend pos=north west, title={$r_2=5, r_1=1$} 
				]    
				\addplot [  
				mark=*, mark size=1pt,  color= red
				] 
				table 
				[ 
				x expr=\thisrowno{0}, y expr=\thisrowno{6}/100 
				] {sample3.txt}; 
				\addplot [  
				mark=triangle*, mark size=1pt,  color= blue
				] 
				table 
				[ 
				x expr=\thisrowno{0}, y expr=\thisrowno{7}/100 
				] {sample3.txt};
				\addplot [  
				mark=square*, mark size=1pt,  color= yellow
				] 
				table 
				[ 
				x expr=\thisrowno{0}, y expr=\thisrowno{8}/100  
				] {sample3.txt}; 
				\addplot [  
				mark=diamond*, mark size=1pt,  color= black
				] 
				table 
				[ 
				x expr=\thisrowno{0}, y expr=\thisrowno{9}/100 
				] {sample3.txt};        
			\end{axis} 
		\end{tikzpicture}
 \end{subfigure}\hfill%
	
 \begin{subfigure}[b]{1.0\textwidth} 
		\begin{tikzpicture} 
			\begin{axis}[ 
				xlabel=$T$, ymax= 7, ylabel=Average Revenue Loss,
				legend pos=north west, title={$r_2=2, r_1=1$}]           
				\addplot [  
				mark=*, mark size=1pt,  color= red
				] 
				table 
				[ 
				x expr=\thisrowno{0}, y expr=\thisrowno{1}/\thisrowno{0}
				] {sample3.txt}; 
				\addplot [  
				mark=triangle*, mark size=1pt,   color= blue
				] 
				table 
				[ 
				x expr=\thisrowno{0}, y expr=\thisrowno{2}/\thisrowno{0} 
				] {sample3.txt};
				\addplot [  
				mark=square*, mark size=1pt,   color= yellow
				] 
				table 
				[ 
				x expr=\thisrowno{0}, y expr=\thisrowno{3}/\thisrowno{0}
				] {sample3.txt}; 
				\addplot [  
				mark=diamond*, mark size=1pt,  color= black
				] 
				table 
				[ 
				x expr=\thisrowno{0}, y expr=\thisrowno{4}/\thisrowno{0} 
				] {sample3.txt}; 
			\end{axis} 
		\end{tikzpicture}   
	\hfill%
		\begin{tikzpicture} 
			\begin{axis}[ 
				xlabel=$T$, ymax= 15, 
				legend pos=north west, title={$r_2=5, r_1=1$}]           
				\addplot [  
				mark=*, mark size=1pt,  color= red
				] 
				table 
				[ 
				x expr=\thisrowno{0}, y expr=\thisrowno{6}/\thisrowno{0}
				] {sample3.txt}; 
				\addplot [  
				mark=triangle*, mark size=1pt,   color= blue
				] 
				table 
				[ 
				x expr=\thisrowno{0}, y expr=\thisrowno{7}/\thisrowno{0} 
				] {sample3.txt};
				\addplot [  
				mark=square*, mark size=1pt,   color= yellow
				] 
				table 
				[ 
				x expr=\thisrowno{0}, y expr=\thisrowno{8}/\thisrowno{0}
				] {sample3.txt}; 
				\addplot [  
				mark=diamond*, mark size=1pt,  color= black
				] 
				table 
				[ 
				x expr=\thisrowno{0}, y expr=\thisrowno{9}/\thisrowno{0} 
				] {sample3.txt}; 
			\end{axis} 
		\end{tikzpicture}   
 \end{subfigure}
 
 \begin{subfigure}[b]{1.0\textwidth} 
 	\begin{tikzpicture} 
 		\begin{axis}[ 
 			xlabel=$\ln T$, ylabel=Total Revenue Loss,
 			ymax= 280, legend pos=north west, title={$r_2=2, r_1=1$} 
 			]    
 			\addplot [  
 			mark=*, mark size=1pt,  color= red
 			] 
 			table 
 			[ 
 			x expr=\thisrowno{5}, y expr=\thisrowno{1}/100 
 			] {sample3.txt}; 
 			\addplot [  
 			mark=triangle*, mark size=1pt,  color= blue
 			] 
 			table 
 			[ 
 			x expr=\thisrowno{5}, y expr=\thisrowno{2}/100 
 			] {sample3.txt};
 			\addplot [  
 			mark=square*, mark size=1pt,  color= yellow
 			] 
 			table 
 			[ 
 			x expr=\thisrowno{5}, y expr=\thisrowno{3}/100  
 			] {sample3.txt}; 
 			\addplot [  
 			mark=diamond*, mark size=1pt,  color= black
 			] 
 			table 
 			[ 
 			x expr=\thisrowno{5}, y expr=\thisrowno{4}/100 
 			] {sample3.txt};        
 		\end{axis} 
 	\end{tikzpicture}\hfill%
 	\begin{tikzpicture} 
 		\begin{axis}[ 
 			xlabel=$\ln T$, 
 			ymax= 580, legend pos=north west, title={$r_2=5, r_1=1$} 
 			]    
 			\addplot [  
 			mark=*, mark size=1pt,  color= red
 			] 
 			table 
 			[ 
 			x expr=\thisrowno{5}, y expr=\thisrowno{6}/100 
 			] {sample3.txt}; 
 			\addplot [  
 			mark=triangle*, mark size=1pt,  color= blue
 			] 
 			table 
 			[ 
 			x expr=\thisrowno{5}, y expr=\thisrowno{7}/100 
 			] {sample3.txt};
 			\addplot [  
 			mark=square*, mark size=1pt,  color= yellow
 			] 
 			table 
 			[ 
 			x expr=\thisrowno{5}, y expr=\thisrowno{8}/100  
 			] {sample3.txt}; 
 			\addplot [  
 			mark=diamond*, mark size=1pt,  color= black
 			] 
 			table 
 			[ 
 			x expr=\thisrowno{5}, y expr=\thisrowno{9}/100 
 			] {sample3.txt};        
 		\end{axis} 
 	\end{tikzpicture}
 \end{subfigure}\hfill%
	\caption{The expected loss under different algorithms} 
	\label{fig:loss} 
\end{figure} 

We make the following observations: 
\begin{enumerate}
	\item The total expected loss under the reoptimization policy grows sublinearly with the horizon length $T$, and appears to scale approximately logarithmically. Although the theoretical upper bound is $O(\ln^3 T)$, the numerical results suggest a much milder growth rate, close to linear in $\ln T$. This is further supported by the approximately linear relationship observed between the expected loss and $\ln T$ across all tested instances. These findings indicate that the policy effectively mitigates the accumulation of errors over time, in contrast to classical heuristics where the loss typically scales on the order of $\sqrt{T}$. 
	\item The average loss per period decreases steadily as $T$ increases and converges to zero. Moreover, this qualitative behavior remains stable across different parameter settings and horizon lengths. This confirms that the policy is asymptotically optimal on a per-period basis and highlights the robustness of the algorithm. The result reflects the diminishing impact of learning errors as more observations are collected and incorporated into the reoptimization process. 
	\item The results demonstrate that reoptimization plays a critical role in controlling capacity allocation, while the learning component governs the residual loss. Continuous reoptimization prevents the buildup of capacity mismatch, and the remaining loss is primarily due to estimation uncertainty. This explains why the total loss grows only logarithmically rather than at a polynomial rate.
\end{enumerate} 


\section{Concluding Remarks}
\label{sec:con}
In this work, we develop a new class of reoptimization-based algorithms for Contextual Bandits with Knapsack (CBwK), a fundamental model for sequential decision-making under resource constraints. By integrating upper-confidence-bound (UCB) learning with dynamically updated capacity allocation, our approach provides a simple and implementable framework that effectively balances exploration and exploitation in environments with heterogeneous customers, products, and limited resources. 
A key innovation of our method is the use of repeated reoptimization to continually recalibrate decisions to the remaining capacity, which enables the algorithm to adapt to stochastic arrivals and noisy rewards while maintaining feasibility at every step.

Our theoretical contribution is the derivation of an average regret bound of $O(\frac{(\ln T)^3}{T})$ which significantly improves upon the 
$O(\frac{1}{\sqrt{T}})$ rates that characterize prior reoptimization-based methods in related dynamic-pricing and resource-allocation settings. This demonstrates the power of reoptimization not only as a heuristic—as often used in practice—but also as a tool that can be rigorously analyzed to achieve provably near-optimal performance. Our numerical experiments further highlight the importance of re-solving: algorithms without reoptimization or with separated learning phases perform substantially worse, validating the structural advantages of our online reoptimization design.

There are several directions for future work that we think would be valuable. 
First, extending our reoptimization methodology to nonlinear reward structures, nonparametric models, or contextual bandits with richer action spaces may broaden its applicability. 
Second, mathematical generality is indeed limited by Assumption 1. 
Unfortunately, the techniques in this paper do not allow us to relax this assumption. We leave this challenging extension to future work.



\bibliography{bib}
\bibliographystyle{chicago}

\appendix
\section{Proofs}
\begin{proof}[Proof of Theorem 1:]
	Recall the LP of (10) and (15) in the following,
	\begin{eqnarray*}
		\label{eq:KKT-ep}
		\min                 && -\sum\limits_{i=1}^I\sum_{j=1}^J x_{ij} {U_i}^T Y_j,\\  \label{eq:tight}
		\textrm{ subject to }&& \sum\limits_{i} x_{ij}  a_{jk} \le b_k,\quad k=1,\ldots,K,\\ \nonumber
		&& \sum_{j} x_{ij} \le \lambda_i, \quad i=1,\ldots,I,\\ \nonumber
		&& x \ge 0.
	\end{eqnarray*}
	The Lagrangian functions of them are  
	\begin{align*}L(\mu,\lambda,x)=-\sum\limits_{i=1}^I\sum_{j=1}^J x_{ij} {U_i}^T Y_j + \sum\limits_{k=1}^K \mu_k(\sum\limits_{i} x_{ij}  a_{jk}-b_k)+ \sum\limits_{i=1}^I (\sum_{j} x_{ij}-\lambda_i),
	\end{align*} 
	with $\mu,\lambda,x \ge 0.$
	Reordering in terms of $x_{ij}$ in $L(\mu,\lambda,x)$ derives that $-{U_i}^T Y_j+\sum\limits_{k=1}^K \mu_k a_{jk}+\lambda_i \le 0$ for any $x_{ij}$. 
	So we kave $\mu_k \le \frac{{U_i}^T Y_j}{a_{jk}}$ for any $x_{ij}$. thus $\mu_k \le \frac{U_i^T Y_j}{a_{jk}}$.
	Let $C_{\mu}=\max\limits_{k} \{\min\limits_{i,j} \{\frac{U_i^T Y_j}{a_{jk}}\}\}$, then $C_{\mu}$ is an constant upper bound for each $\mu_k$ under any $b'$ satisfying $\|b'-B/T\|_\infty <\epsilon$.
	Noting that $\mu=r^T W^{-1}$, so whenever $\|b'-B/T\|_\infty <\epsilon$,
	we have that $r^T W^{-1} (b'-B/T)=\mu(b'-B/T) \le k c_{\mu} \epsilon$. 
	This shows that the $\epsilon$ can be chosen at least at $\frac{\Delta}{ k C_{\mu}}$, which is in order of $\Theta(\Delta)$. 
\end{proof}

\begin{definition}
	We define $w_{it}$ in the below to bound the pseudo regret of a type-$i$ customer arriving in period $t$, i.e.,
	\begin{align*}
		w_{it}=\left\{ \begin{array} {ll}
			\sqrt{Y_{j(t)}^T M_{i(t)t}^{-1} Y_{j(t)}}, & i=i(t),\\
			0, & o.w.
		\end{array}\right.
	\end{align*}
\end{definition}

\begin{corollary}
	For all $i=1,\ldots,I$ and $t=1,\ldots,T$, $\det(M_{it+1}) =\prod\limits_{s=1}^{t} (1+w_{is}^2)$.
\end{corollary}

\begin{proof}[Proof of Corollray 1:]
	In this proof, we omit the superscript $i$ for simplicity.  
	We use $Y_t=Y_{j(t)}$ for short.
	By (8), we have 
	\begin{eqnarray}
		\det(M_{t+1}) \nonumber
		&=&\det(M_t+Y_tY_t^T)\\\nonumber
		&=&\det(M_t^{1/2}(I+M_t^{-1/2}Y_tY_t^TM_t^{-1/2})M_t^{1/2})\\\nonumber
		&=&\det(M_t)\det(I+M_t^{-1/2}Y_t(Y_tM_t^{-1/2})^T)\\\nonumber
		&=&\det(M_t)\det(I+v_tv_t^T), \nonumber
	\end{eqnarray}
	where $v_t=M_t^{-1/2}Y_t$.
	Note that $1+w_t^2$ is a eigenvalue of $I+v_tv_t^T$, since
	\begin{align*}
		(I+v_tv_t^T)v_t=v_t+v_t(v_t^Tv_t)=v_t(1+Y_t^T M_t^{-1} Y_t)=(1+w_t^2)v_t.
	\end{align*}
	Since $v_tv^T_t$ is a rank-one matrix, all the other eigenvalues of $I+v_tv_t^T$ equal $1$. 
	It follows that $\det(I+v_tv_t^T)=1+w_t^2$.
	Recalling that $M_1$ is the identity matrix, the result follows by induction. 
\end{proof}

\begin{proof}[Proof of Theorem 2:]
	In this proof, we omit the superscript $i$ for simplicity.  
	Also, we will use $Y_t=Y_{j(t)}$ for short.
	\begin{eqnarray}
		\mathrm{Tr}(M_t)\nonumber
		&=&\mathrm{Tr}(I+\sum\limits_{s<t} Y_sY_s^T)\\\nonumber
		&=&\mathrm{Tr}(I)+\sum\limits_{s<t} \mathrm{Tr}(Y_sY_s^T)\\\nonumber
		&=&J+\sum\limits_{s<t} ||Y_s||_2\\\nonumber
		&\le& Jt. \nonumber
	\end{eqnarray}
	The trace of $M_t$ equals the sum of the eigenvalues of $M_t$. 
	Note that $M_t$ is always positive definite, and hence its eigenvalues are all positive. 
	Subject to these constraints, $\det(M_t)$ is maximized when all the eigenvalues are equal.  In this case, $\det(M_t)$ is at most $t^J$. 
\end{proof}

\begin{proof}[Proof of Theorem 3:]
	Note that when $y \le 1$, we have that $y \le 2\ln(1+y)$.
	Hence, $\min(w_{ijs}^2,1) \le 2\ln(1+w_{ijs}^2)$ which induces that
	\begin{eqnarray}
		\sum\limits_{s=1}^{t} \sum\limits_{i,j} x_{ijs} \min(w_{ijs}^2,1)\nonumber
		&\le&2 \sum\limits_{s=1}^{t} \sum\limits_{i,j} x_{ijs} \ln(1+w_{ijs}^2)\\\nonumber
		&=&2 \E \sum\limits_{i=1}^{I} \ln(\det(M_{it})) \\\nonumber
		&\le& 2 IJ\ln(t), \nonumber
	\end{eqnarray}
	the second line uses Corollray 1 and the last inequality uses Lemma 2.
	
	By the Cauchy-Schwartz inequality and Proposition 6, we have that
	\begin{align*}
		g_s^2 \le  (2 \sqrt{\beta_s} \sum\limits_{i,j} x_{ijs} \min(w_{ijt},1))^2 
		\le  4 \beta_s (\sum\limits_{i,j} x_{ijs} \min(w_{ijs}^2,1)).
	\end{align*}
	Then \begin{eqnarray}
		\sum\limits_{s=1}^{t} g_s^2 \nonumber
		&\le&\sum\limits_{s=1}^{t} 4\beta_s(\sum\limits_{i,j} x_{ijs} \min(w_{ijs}^2,1))\\\nonumber
		&\le&4\beta_t \sum\limits_{s=1}^{t} \sum\limits_{i,j} x_{ijs} \min(w_{ijs}^2,1)\\\nonumber
		&\le&4\beta_t IJ\ln(t).\nonumber
	\end{eqnarray} 
\end{proof}

\begin{proof}[Proof of Theorem 4:]
	By the updating rule (9) and $\xi_t=r_t-{U_i}^T Y_{j(t)}$, we have that
	\begin{eqnarray}
		z_{it} \nonumber
		&=&M_{it}(\hat{U}_{it}- U_i)  \\\nonumber
		&=&M_{it}(M_{it}^{-1} \sum\limits_{s=1}^{t} r_s Y_{j(s)} \id{i(s)=i} - U_i) \\\nonumber
		&=&\sum\limits_{s=1}^{t} r_s Y_{j(s)} \id{i(s)=i} - M_{it} U_i \\\nonumber
		&=&\sum\limits_{s=1}^{t} r_s Y_{j(s)} \id{i(s)=i} - (I + \sum\limits_{s=1}^{t}  Y_{j(s)}Y_{j(s)}^T \id{i(s)=i}) U_i \\\nonumber
		&=&\sum\limits_{s=1}^{t} (r_s-{U_i}^T Y_{j(s)}) Y_{j(s)} \id{i(s)=i} - U_i \\\nonumber
		&=&\sum\limits_{s=1}^{t} \xi_s Y_{j(s)} \id{i(s)=i}-U_i. \nonumber
	\end{eqnarray} 
\end{proof}

\section{Online Updating without re-solving (UCB) Algorithm}
\label{sc:non-UCB}
\begin{itemize}
	\item Initialization: Set $b_{k}=\frac{B_k}{T}$, $M_{i1}=I_{J \times J}$ and $\hat{U}_{i1}=0$, $i=1,2,\cdots,I$, $k=1,\ldots,K$.
	\item For $t=1$ to $T$:
	\begin{eqnarray*}
		\textrm{CB}_{it}&=&\{v: \|v-\hat{U}_{it} \|_{2,M_{it}} \le \sqrt{\beta_t} \},\quad \forall i\in[1,I],\\
		\textrm{CB}_t&=&\prod\limits_{i=1}^I \textrm{CB}_{it}.
	\end{eqnarray*}
	Note that $\textrm{CB}_t$ is the product of $I$ intervals.
	
	\item In each period $t$, solve the following optimization problem to obtain probabilities $x^*_t$.
	\begin{align*}
		\label{equ:max-potential_OFF}
		x^*_t =\argmax\limits_{x\in D_t} \max\limits_{V \in \textrm{CB}_t} \sum_{i,j} x_{ij} V_i^T Y_j,
	\end{align*}
	where, 
	\begin{align*}
		D_t=\{x\ |\ \sum_{i=1}^I \sum_{j=1}^J x_{ij} a_{jk} \le b_{k},\ \forall k\in[1,K];\ \sum\limits_{j} x_{ij} \le \lambda_i,\ \forall i\in[1,I];\ x^* \ge 0\}
	\end{align*}
	is the feasible region.   
	The inner ``max'' means that we choose the best potential reward based on all the possible vectors in $\textrm{CB}_t$. 
	
	\item Use the solution $x^*_t$ to assign the customer at $t$, if any, to a resource and collect the reward $r_t$.
	
	\item Update the estimator $\hat{U}_{it+1}$ and matrix $M_{it+1}$ via
	\begin{eqnarray}
		M_{it+1}&=&\left\{ 
		\begin{array} {ll}
			M_{it}+ Y_{j(t)} Y_{j(t)}^T,\quad i=i(t),\\ \nonumber
			M_{it},\quad \mbox{o.w.}
		\end{array}\right. \label{equ:update_M_OFF}\\ \nonumber
		\hat{U}_{it+1}&=&\left\{ 
		\begin{array} {ll}
			M_{it+1}^{-1} \sum\limits_{s=1}^t r_s Y_{j(s)}\id{i(s)=i} ,\quad i=i(t),\\
			\hat{U}_{it},\quad o.w.
		\end{array}\right.  \label{equ:update_XA_OFF}
	\end{eqnarray}
\end{itemize}

\section{Separated learning and exploitation with re-solving (Re-SEP) algorithm}
\label{sc:Re-SEP}
In this algorithm we have two stage: the exploration stage and the exploitation stage.
Let $L$ be the length of the exploration stage and we will choose a proper $L$ later.

\begin{itemize}
	\item \textbf{Stage 1 (Exploration)}
	\begin{enumerate}
		\item Assign exploration match: assign the $i$-th type customer to resource from $1$ to $J$ sequentially and periodically. 
		\item Set $\hat{U}_{i}=0$ and $M_{i}=I_{J \times J}$ for $i=1,\ldots,I$.
		\item For $t=1$ to $L$, do: Collect reward $r_t$, and update the estimator $\hat{U}_{i}$ and matrix $M_{i}$ as 	
		\begin{eqnarray}
			\hat{U}_{i}&=&
			\left\{ 
			\begin{array} {ll}
				(M_{i}+ Y_{j(t)} Y_{j(t)}^T)^{-1} (M_{i} \hat{U}_{i} + r_t Y_{j(t)}), \quad i=i(t),\\ \nonumber
				\hat{U}_{i}, \quad \mbox{o.w.} 
			\end{array} \right. \\ \nonumber
			M_{i}&=&\left\{ 
			\begin{array} {ll}
				M_{i}+ Y_{j(t)} Y_{j(t)}^T, \quad i=i(t),\\ \nonumber
				M_{i},\quad \mbox{o.w.} 
			\end{array} \right. 
		\end{eqnarray}
	\end{enumerate}
	
	\item \textbf{Stage 2 (Exploitation)}
	\begin{enumerate}
		\item Count each amount $b_k$ of type-$k$ of resource used in Stage 1 and reset the average $b_{kL}=\frac{B_k-b_k}{T-L}$ for $k=1,\ldots,K$.
		\item For $t=L+1$ to $T$, do: 
		\begin{enumerate}
			\item Solve the following optimization problem to obtain probabilities $x^*_t$ in period $t$
			\begin{align*}
				x^*_t =\argmax\limits_{x\in D_t} \sum_{i,j} x_{ij} \hat{U}_i^T Y_j,
			\end{align*}
			where, $D_t=\{x\ |\ \sum_{i=1}^I \sum_{j=1}^J x_{ij} a_{jk} \le b_{kt},\ \forall k\in[1,K];\ \sum\limits_{j} x_{ij} \le \lambda_i,\ \forall i\in[1,I];\ x \ge 0\}$ is the feasible region.  
			
			\item Update the capacity allocated to period $t+1$, 
			\begin{align*}
				b_{k t+1}=b_{kt} - \frac{(d_{kt}-b_{kt})}{T-t},\quad \forall k\in[1,K].
			\end{align*} 
			
			\item Use the solution $x^*_t$ to assign the customer at $t$, if any, to a resource and get reward $r_t$.
		\end{enumerate}
	\end{enumerate}
\end{itemize}
The above algorithm give you a total reward $R=\sum\limits_{t=1}^T r_t$ based on the realization of solution $x^*_t$.
According to \cite{CJD2018}, an optimal tuning parameter for $L$ is to set $L=\sqrt{kT}$, where $k$ is a fixed number.

\section{Separated learning and exploitation without re-solving (SEP) algorithm}
\label{sc:SEP}
In this algorithm we have two stage: the exploration stage and the exploitation stage.
Let $L$ be the length of the exploration stage and we will choose a proper $L$ later.

\begin{itemize}
	\item \textbf{Stage 1 (Exploration)}
	\begin{enumerate}
		\item Assign exploration match: assign the $i$-th type customer to resource from $1$ to $J$ sequentially and periodically. 
		\item Set $\hat{U}_{i}=0$ and $M_{i}=I_{J \times J}$ for $i=1,\ldots,I$.
		\item For $t=1$ to $L$, do: Collect reward $r_t$, and update the estimator $\hat{U}_{i}$ and matrix $M_{i}$ as 	
		\begin{eqnarray}
			\hat{U}_{i}&=&
			\left\{ 
			\begin{array} {ll}
				(M_{i}+ Y_{j(t)} Y_{j(t)}^T)^{-1} (M_{i} \hat{U}_{i} + r_t Y_{j(t)}), \quad i=i(t),\\ \nonumber
				\hat{U}_{i}, \quad \mbox{o.w.} 
			\end{array} \right. \\ \nonumber
			M_{i}&=&\left\{ 
			\begin{array} {ll}
				M_{i}+ Y_{j(t)} Y_{j(t)}^T, \quad i=i(t),\\ \nonumber
				M_{i},\quad \mbox{o.w.} 
			\end{array} \right. 
		\end{eqnarray}
	\end{enumerate}
	
	\item \textbf{Stage 2 (Exploitation)}
	\begin{enumerate}
		\item Count each amount $b_k$ of type-$k$ of resource used in Stage 1 and reset the average $b_{kL}=\frac{B_k-b_k}{T-L}$ for $k=1,\ldots,K$.
		\item For $t=L+1$, solve the following optimization problem to obtain probabilities $x^*$
		\begin{align*}
			x^* =\argmax\limits_{x\in D_t} \sum_{i,j} x_{ij} \hat{U}_i^T Y_j,
		\end{align*}
		where, $D=\{x\ |\ \sum_{i=1}^I \sum_{j=1}^J x_{ij} a_{jk} \le b_{kL},\ \forall k\in[1,K];\ \sum\limits_{j} x_{ij} \le \lambda_i,\ \forall i\in[1,I];\ x \ge 0\}$ is the feasible region.  
		
		\item Use the solution $x^*$ to assign the customer at $t$, if any, to a resource and get reward $r_t$.
	\end{enumerate}
\end{itemize}
We still choose $L=\sqrt{kT}$, where $k$ is a fixed number.

\end{document}